\documentclass[letterpaper,twocolumn,10pt]{article}
\usepackage{usenix2019_v3}

\usepackage{graphicx} 
\usepackage{amsmath}
\usepackage{amsfonts}
\usepackage{tikz}
\usepackage{amsthm}
\usepackage{thmtools}
\usepackage{todonotes}
\usepackage{booktabs}
\usepackage{soul,color}
\usepackage{multirow}
\usepackage{subfig}
\usepackage{caption}

\declaretheorem[name=Definition]{definition}
\declaretheorem[name=Problem]{problem}
\declaretheorem[name=Theorem]{theorem}

\newcommand{\R}{\mathbb{R}}
\newcommand{\N}{\mathbb{N}}
\newcommand{\feats}{\mathcal{X}}
\newcommand{\labels}{\mathcal{Y}}
\newcommand{\dataset}{\mathcal{D}}
\newcommand{\dtrain}{\dataset_{\textit{train}}}
\newcommand{\dtest}{\dataset_{\textit{test}}}

\newcommand{\spread}{\psi_p}

\newcommand{\revise}[1]{#1}
\newcommand{\dist}{\textit{dist}}
\newcommand{\pred}[1]{\widehat{#1}}
\newcommand{\tool}{CARVE-GBM}
\newcommand{\extendedN}{\N \cup \{0\}}

\title{Verifiable Boosted Tree Ensembles}
\author{Anonymous submission}
\date{}

\usepackage{todonotes}

\begin{document}

\author{
{\rm Stefano Calzavara$^{*}$}\\
Università Ca' Foscari Venezia\\
stefano.calzavara@unive.it
\and
{\rm Lorenzo Cazzaro$^{*}$}\\
Università Ca' Foscari Venezia\\
lorenzo.cazzaro@unive.it
\and
{\rm Claudio Lucchese$^{*}$}\\
Università Ca' Foscari Venezia\\
claudio.lucchese@unive.it
\and
{\rm Giulio Ermanno Pibiri$^{*}$}\\
Università Ca' Foscari Venezia\\
giulioermanno.pibiri@unive.it
}

\maketitle
{\let\thefootnote\relax\footnote{{$^*$Equal contribution.}}}
\begin{abstract}
Verifiable learning advocates for training machine learning models amenable to efficient security verification. Prior research demonstrated that specific classes of decision tree ensembles -- called \emph{large-spread ensembles} -- allow for robustness verification in polynomial time against any norm-based attacker. This study expands prior work on verifiable learning from basic ensemble methods
(i.e., hard majority voting) to advanced \textit{boosted tree ensembles}, such as those trained using XGBoost or LightGBM. Our formal results indicate that robustness verification is achievable in polynomial time when considering attackers based on the $L_\infty$-norm, but remains NP-hard for other norm-based attackers. Nevertheless, we present a pseudo-polynomial time algorithm to verify robustness against attackers based on the $L_p$-norm for any $p \in \extendedN$, which in practice grants excellent performance. Our experimental evaluation shows that large-spread boosted ensembles are accurate enough for practical adoption, while being amenable to efficient security verification.
\end{abstract}

\section{Introduction}
Security of Machine Learning (ML) is a hot topic nowadays, because models trained using classic supervised learning algorithms proved vulnerable to \emph{evasion attacks}, i.e., malicious perturbations of inputs designed to force mispredictions at test time~\cite{BiggioCMNSLGR13,SzegedyZSBEGF13,DemetrioCBLAR21}. When ML models are deployed in adversarial settings, standard performance measures such as accuracy, precision and recall do not provide appropriate guarantees, because they do not take adversarial perturbations into account. This motivated a long research line on adversarial ML and the definition of new measures such as \emph{robustness}, which explicitly quantifies resistance to evasion attacks~\cite{MadryMSTV18}.

Unfortunately, verifying the security of ML models against evasion attacks is computationally hard, because verification must consider all the possible adversarial perturbations that the attacker may perform. In this work, we focus on the security of \emph{tree ensembles}~\cite{BreimanFOS84}, a popular class of ML models particularly effective for non-perceptual classification tasks.
Kantchelian et al.~\cite{KantchelianTJ16} were the first to prove that the robustness verification problem for tree ensembles is NP-complete when malicious perturbations are modeled by a norm. Follow-up work~\cite{WangZCBH20} extended this negative result to stump ensembles, i.e., ensembles including just trees of depth one, and proposed \emph{approximate} verification approaches, which can formally prove the absence of evasion attacks, but may incorrectly report evasion attacks also for secure inputs. Exact verification approaches against specific attackers, e.g., modeled in terms of the $L_\infty$-norm, have also been proposed~\cite{ChenZS0BH19,RanzatoZ20}. Yet, such approaches have to deal with the NP-hardness of robustness verification and are doomed to fail when the size of the tree ensembles is large.

To improve over this bleak picture, recent work proposed \emph{verifiable learning} for tree ensembles~\cite{CalzavaraCPP23}. The key idea of verifiable learning is the development of new training algorithms that learn restricted classes of models amenable for efficient security verification, e.g., in polynomial time. Although promising, prior work is still limited in scope because it assumes the adoption of simple ensemble methods based on hard majority voting, i.e., each tree in the ensemble makes a class prediction and the most frequent class is returned by the ensemble. State-of-the-art ensemble methods such as \emph{gradient boosting}~\cite{friedman2001greedy} operate rather differently, because each tree in the ensemble is a \emph{regressor} predicting a real-valued score; the ensemble prediction is then performed by summing together the individual scores and translating the result into the final class prediction. This makes it difficult to generalize existing work on verifiable learning to boosted tree ensembles.

\paragraph*{Contributions.}
We here summarize our contributions:
\begin{enumerate}
    \item We extend existing research on verifiable learning~\cite{CalzavaraCPP23} from simple ensemble methods based on hard majority voting to state-of-the-art boosted tree ensembles, e.g., those trained using LightGBM~\cite{lightgbm}. Our analysis shows that a restricted class of tree-based models, called \emph{large-spread boosted ensembles}, admit exact security verification in polynomial time when considering attacks based on the $L_\infty$-norm. We then prove a negative result: security verification remains NP-hard even for our restricted class of models when considering other norm-based attackers. Still, we present a pseudo-polynomial time algorithm to verify robustness against attackers based on the $L_p$-norm for any $p \in \extendedN$, which in practice grants excellent performance (Sections~\ref{sec:verification} and~\ref{sec:solving}).
    
    \item We implement our efficient verification algorithms for large-spread boosted ensembles and we propose a new training algorithm for such models, deployed as a simple extension of the popular LightGBM library. To support reproducible research, we will make our software publicly available on GitHub upon paper acceptance (Section~\ref{sec:implementation}).
    
    \item We perform an extensive experimental evaluation on public datasets to show the accuracy and robustness of large-spread boosted ensembles with respect to traditional boosted ensembles. The net result is that our models are accurate and robust enough for practical adoption, while being amenable to efficient security verification. Moreover, we show that existing verification algorithms for standard boosted ensembles are much less efficient than our algorithms, which enjoy a speedup of up two orders of magnitude, and they may seriously struggle against the largest models, even when provided with significant computational resources (Section~\ref{sec:experiments}).
\end{enumerate}
\section{Background}
We here introduce the key technical ingredients required to appreciate the paper. For readability, Table~\ref{tab:notation} summarizes the main notation used in the paper.


\begin{table}[t]
    \centering
    \begin{tabular}{c|l}
    \toprule
    $\vec{x}$ &  Instance drawn from the feature space $\feats$ \\
    $x_i$ & $i$-th component of the vector $\vec{x}$ \\
    $y$ & Class label drawn from the set of labels $\labels$ \\
    $d$ & Number of features of $\vec{x}$ (i.e., dimensionality of $\feats$) \\
    $t$ & Regression tree \\
    $T$ & Boosted tree ensemble \\
    $N$ & Number of nodes of a boosted tree ensemble \\
    $m$ & Number of trees of a boosted tree ensemble \\
    $\vec{\delta}$ & Adversarial perturbation \\
    $A_{p,k}$ & Attacker based on $L_p$-norm (max perturbation $k$) \\
    \bottomrule
    \end{tabular}
\caption{Summary of notation.}
\label{tab:notation}
\end{table}

\subsection{Supervised Learning}
Let $\feats \subseteq \R^d$ be a $d$-dimensional vector space of real-valued \textit{features}. An \emph{instance} $\vec{x} \in \feats$ is a $d$-dimensional feature vector $\langle x_1, x_2, \ldots, x_d \rangle$ representing an object in the vector space $\feats$. Each instance is assigned a class label $y \in \labels$ by an unknown \emph{target} function $f: \feats \rightarrow \labels$. In this work, we focus on binary classification, i.e., we let $\labels = \{+1,-1\}$ include just a positive and a negative class. Multi-class classification problems can be encoded in terms of binary classification by using standard techniques like one-vs-one and one-vs-rest.

Supervised learning algorithms automatically learn a \emph{classifier} $g: \feats \rightarrow \labels$ from a \emph{training set} of correctly labeled instances $\dtrain = \{(\vec{x}_i,f(\vec{x}_i))\}_i$, with the goal of approximating the target function $f$ as accurately as possible. The performance of classifiers is normally estimated on a \emph{test set} of correctly labeled instances $\dtest = \{(\vec{z}_i,f(\vec{z}_i))\}_i$, disjoint from the training set, yet drawn from the same data distribution. For example, the standard \emph{accuracy} measure $a(g,\dtest)$ counts the percentage of test instances where the classifier $g$ returns a correct prediction.

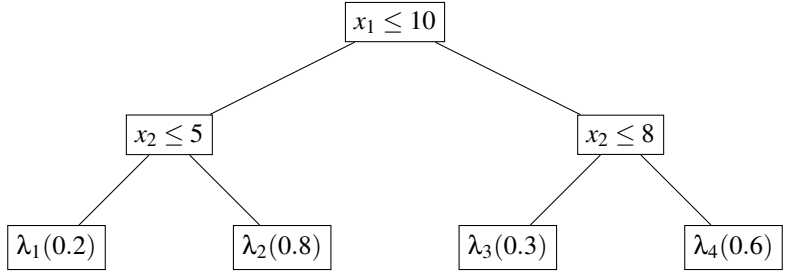
\begin{figure*}[t]
\centering
\begin{tikzpicture}[level 1/.style={sibling distance=6cm},level 2/.style={sibling distance=3cm}]
\tikzstyle{every node}=[rectangle,draw]
\node{$x_1 \leq 10$}
	child { node {$x_2 \leq 5$}
	        child { node {$\lambda_1(0.2)$}}
	        child { node {$\lambda_2(0.8)$}} }
	child { node {$x_2 \leq 8$}
		    child { node {$\lambda_3(0.3)$}}
	        child { node {$\lambda_4(0.6)$}} }
;
\end{tikzpicture}
\caption{Example of regression tree.}
\label{fig:tree}
\end{figure*}

\subsection{Boosted Tree Ensembles}
\label{sec:ensembles}
A \emph{regression tree} $t: \feats \rightarrow \R$ can be inductively defined as follows: $t$ is either a leaf $\lambda(s)$ for some real-valued score $s \in \R$ or an internal node $\sigma(f,v,t_l,t_r)$, where $f \in \{1,\ldots,d\}$ identifies a feature, $v \in \R$ is a threshold for the feature, and $t_l,t_r$ are regression trees (left and right child). At test time, the instance $\vec{x}$ traverses the regression tree $t$ as follows: starting from the root of $t$, for each traversed tree node $\sigma(f,v,t_l,t_r)$, $\vec{x}$ falls into the left sub-tree $t_l$ if $x_f \leq v$ and into the right sub-tree $t_r$ otherwise, until it eventually reaches a leaf $\lambda(s)$. We write $t(\vec{x}) = s$ when $\vec{x}$ reaches a leaf $\lambda(s)$ of the tree $t$ upon prediction and we refer to the score $s$ as the \emph{raw prediction} of $t$ on $\vec{x}$. Raw predictions might have multiple interpretations, representing, e.g., the probability of belonging to the positive class or the predicted value in case of a regression task. For example, Figure~\ref{fig:tree} represents a regression tree $t$ of depth 2 where scores range in the interval [0,1] to represent the probability of belonging to the positive class. In this case $t(\langle 8,6 \rangle) = 0.8$ and $t(\langle 12,7 \rangle) = 0.3$.

A \emph{boosted tree ensemble} $T: \feats \rightarrow \labels$ is a classifier built on top of a set of regression trees $\{t_1,\ldots,t_m\}$, which aggregates individual raw predictions to produce a single class prediction $T(\vec{x})$. We use the term \emph{boosted} to stress that these ensembles are assumed to have been trained using state-of-the-art boosting algorithms like AdaBoost~\cite{adaboost}, Gradient Boosting~\cite{friedman2001greedy}, and its popular variants such as LightGBM~\cite{lightgbm} and XGBoost~\cite{xgboost}; for readability, we often use the terms ``tree ensembles'' or even just ``ensembles'' in the following. 

Given an instance $\vec{x}$, the ensemble $T$ computes the class prediction $T(\vec{x})$ as follows. First, the ensemble computes the raw prediction $\pred{T}(\vec{x}) = \sum_{i=1}^m t_i(\vec{x})$. The raw prediction $\pred{T}(\vec{x})$ is then transformed by an \emph{inverse link} function $\iota: \R \rightarrow \R$, e.g., the \emph{logit} function when individual tree scores are probabilities, and compared against a threshold $\tau \in \R$: if $\iota(\pred{T}(\vec{x})) \geq \tau$ then $T(\vec{x}) = +1$, otherwise $T(\vec{x}) = -1$. We do not assume any specific choice of $\iota$, but we require $\iota$ to be monotonically increasing, i.e., higher scores of the raw prediction push the prediction towards the positive class.

\subsection{Classifier Robustness}
Robustness is a popular measure used to estimate the performance of classifiers deployed in an adversarial setting. It requires the classifier to perform a correct prediction on a test instance $\vec{x}$ and stick to the same prediction for any possible \emph{evasion attack} attempt crafted from $\vec{x}$, e.g., by adding some maliciously crafted perturbation $\vec{\delta} \in \R^d$ to it. We model the \emph{attacker} $A: \feats \rightarrow 2^{\feats}$ as a function from instances to sets of instances, defining the possible evasion attacks against them. We assume that $\vec{x} \in A(\vec{x})$ for all instances $\vec{x} \in \feats$, i.e., the attacker can always leave the original instances unchanged.

\begin{definition}[Robustness]
The classifier $g$ is \emph{robust} against the attacker $A$ on the instance $\vec{x}$ with true label $y$ if and only if $\forall \vec{z} \in A(\vec{x}): g(\vec{z}) = y$.
\end{definition}

Based on the definition of robustness, for a given attacker $A$, we can define the robustness measure $r_A(g,\dtest)$ by computing the percentage of test instances where the classifier $g$ is robust. In the following, we focus on attackers represented in terms of an arbitrary norm, i.e., the attacker's capabilities are defined by some norm function and a maximum perturbation $k$. Concretely, we consider the attacker defined as $A_{p,k}(\vec{x}) = \{\vec{z} \in \feats ~|~ ||\vec{z} - \vec{x}||_p \leq k\}$ for some budget $k$ and (using a little abuse of notation)
$p \in \N \cup \{0,\infty\}$.

\subsection{Robustness Verification of Tree Ensembles}
\label{sec:verification_decision_trees}
The robustness verification problem for tree ensembles is NP-hard for any norm-based attacker~\cite{KantchelianTJ16}. Although different heuristics have been proposed to tame this computational complexity~\cite{KantchelianTJ16,ChenZS0BH19,WangZCBH20}, NP-hardness still constitutes a roadblock to verification when the model size grows.

To cope with this complexity,
recent work identified a restricted class of tree ensembles -- known as \emph{large-spread ensembles} -- admitting robustness verification in polynomial time~\cite{CalzavaraCPP23}. This positive result assumes ensembles based on hard majority voting, where each tree makes its own class prediction and the ensemble returns the most frequently predicted class. As we explained in Section~\ref{sec:ensembles}, the boosted tree ensembles considered in this paper operate rather differently, because they add together raw predictions (scores) and use an inverse link function $\iota$ to determine the class prediction through thresholding. 

The key characteristic of large-spread ensembles is that the thresholds chosen for different trees are sufficiently far away that each feature can be successfully attacked in at most one tree, hence the attacks against a large-spread ensemble can be decomposed into a sum of orthogonal attacks against the individual trees. This permits to compose the security analysis of the individual trees to draw conclusions about the robustness of the entire ensemble, thus enabling efficient verification. The formal definition of large-spread ensemble is given below.

\begin{definition}[Large-Spread Ensemble~\cite{CalzavaraCPP23}]
\label{def:spread}
Given the ensemble $T = \{t_1, \dots, t_m\}$, its $p$-spread $\spread(T)$ is the minimum value $||v-v'||_p$ computed for any $v,v'$ such that there exists two different trees $t,t' \in T$ such that $\sigma(f,v,t_l,t_r) \in t$ and $\sigma(f,v',t_l',t_r') \in t'$ for some $f,t_l,t_r,t_l',t_r'$. We say that $T$ is \emph{large-spread} for the attacker $A_{p,k}$ iff $\spread(T) > 2k$.
\end{definition}

The existing verification algorithm for large-spread ensembles is based on a tree annotation procedure, which we also leverage in this paper. The annotation procedure associates each node of the individual trees with a symbolic representation of the set of instances that may traverse it in presence of adversarial manipulations. The procedure first annotates the root of the tree with the $d$-dimensional hyper-rectangle $(-\infty,+\infty]^d$, meaning that every instance will traverse the root. Children are then annotated by means of a recursive tree traversal: concretely, if the parent node $\sigma(f,v,t_1,t_2)$ is annotated with
$(l_i,r_i]^d$,
then the annotations of the roots of $t_1$ and $t_2$ are defined as
$(l_i^1,r_i^1]^d$
and
$(l_i^2,r_i^2]^d$
respectively:
\begin{equation*}
(l_i^1,r_i^1] = \begin{cases}
(l_i,r_i] \cap (-\infty,v] = (l_i,\min\{r_i,v\}] & \textnormal{if } i = f \\
(l_i,r_i] & \textnormal{otherwise},
\end{cases}
\end{equation*}
and:
\begin{equation*}
(l_i^2,r_i^2] = \begin{cases}
(l_i,r_i] \cap (v,+\infty) = (\max\{l_i,v\},r_i] & \textnormal{if } i = f \\
(l_i,r_i] & \textnormal{otherwise}.
\end{cases}
\end{equation*}

Given an annotated tree and an instance $\vec{x}$, it is possible to identify the minimal adversarial perturbation required to push $\vec{x}$ into any given leaf $\lambda(s)$. In particular, let $H = (l_1,r_1] \times \ldots \times (l_d,r_d]$ be the hyper-rectangle annotating $\lambda(s)$, then we define $\textit{dist}(\vec{x},\lambda(s)) = \vec{\delta} \in \R^d$, where:\footnote{We write $l_i - x_i + \varepsilon$ to stand for the minimum floating point number which is greater than $l_i - x_i$. We also assume here that $H$ is not empty, i.e., there does not exist any $(l_j,r_j]$ in $H'$ such that $l_j \geq r_j$.}
\begin{equation*}
\forall i \in [1,d]: \delta_i =
\begin{cases}
0& \textnormal{if } x_i \in H_i=(l_i, r_i] \\
l_i - x_i + \varepsilon & \textnormal{if } x_i \leq l_i \\
r_i - x_i & \textnormal{if } x_i > r_i.
\end{cases}
\end{equation*}

The value $||\vec{\delta}||_p$ is the norm of the minimal perturbation required to push $\vec{x}$ into the leaf $\lambda(s)$. Let then $L(t,\vec{x})$ be the set of the leaves of $t$ that are reachable by $\vec{x}$ as the result of adversarial manipulations, that is:
\[
L(t,\vec{x}) = \{\lambda(s) \in t ~|~ ||\dist(\vec{x}, \lambda(s)) ||_p \leq k\}.
\]

Note that we are only interested in the norms of the adversarial manipulations,
and not on the hyper-rectangles needed to compute them.
By exploiting this fact,
it is possible to compute the set $L(t,\vec{x})$ and $\{ ||\dist(\vec{x},\lambda(s))||_p ~|~ \lambda(s) \in L(t,\vec{x}) \}$
in $O(N)$, i.e., in linear time with respect to the
number of nodes of the tree ensemble~\cite{CalzavaraCPP23}. We assume this complexity for subsequent analyses.

\section{Robustness Verification of Large-Spread Boosted Ensembles}
\label{sec:verification}
Large-spread ensembles drive away from the negative NP-hardness result of robustness verification, however prior work just considers a simple ensemble method based on hard majority voting and does not apply to state-of-the-art boosting schemes, e.g., GBDTs~\cite{CalzavaraCPP23}. Here we investigate the robustness verification problem for large-spread boosted tree ensembles. 

\subsection{Optimization Problem}\label{sec:opt-problem}
The key insight of this work is that, given a large-spread boosted ensemble $T$ and an instance $\vec{x} \in \feats$ with true label $y$, we can identify the optimal evasion attack strategy for the attacker $A_{p,k}$ by solving an optimization problem. In particular, we can identify the least adversarial manipulation $\vec{\delta}$ such that $T(\vec{x} + \vec{\delta}) \neq y$, if any exists. In our formalization, we leverage the insight that any attack against a large-spread ensemble can be decomposed into a sum of orthogonal adversarial perturbations operating against the individual trees $t_i \in T$, hence the optimal evasion attack can be identified by finding the sub-ensemble $T' \subseteq T$ including the best trees to target and adding up their corresponding perturbations. 

Our focus on boosting complicates the formulation with respect to prior work on large-spread ensembles~\cite{CalzavaraCPP23}. In particular, the existing robustness verification algorithm~\cite{CalzavaraCPP23} assumes ensembles based on hard majority voting of the class predictions, hence the best sub-ensemble $T'$ to target just includes those trees requiring the least amount of adversarial perturbation to predict the wrong class. In the case of boosting, each tree outputs a real-valued raw prediction and different attacks lead to different changes to such predictions, making the optimal evasion attack strategy harder to identify. The challenge is that attacks cannot be totally ordered in general, because some attacks require a small perturbation (good for the attacker) but have just a limited impact on the raw prediction (bad for the attacker), while other attacks require a large perturbation (bad for the attacker) but have a large impact on the raw prediction (good for the attacker). For example, consider the regression tree in Figure~\ref{fig:tree} and the instance $\langle 9.1,0.1 \rangle$ with label $-1$. This instance would normally fall in the leaf $\lambda_1(0.2)$ upon prediction, however assume it might be corrupted by the attacker so as to reach any of the other three leaves.\footnote{To make the example more readable, we assume that adversarial perturbations are discrete with a tick $\varepsilon = 0.1$. Our formalization works for real-valued perturbations as required by the considered norms.} To significantly affect the raw prediction, the best choice of the attacker would be pushing the instance into the leaf $\lambda_3(0.8)$, which returns a rather different score; however this requires adding 5 to the second feature. Pushing the instance into the leaf $\lambda_2(0.3)$ has a much lower impact on the original raw prediction, however it just requires adding 1 to the first feature. Although this second attack is less impactful than the first one, it is cheaper and might still be enough to force a prediction error when the tree is used within a boosted ensemble.

The best way to combine different attacks is thus formalized as an optimization problem. For each tree $t_i \in T$ and leaf $\lambda(s_{ij}) \in L(t_i,\vec{x})$, we define the \emph{adversarial gain} of $\lambda(s_{ij})$ as the advantage that the attacker gets when the instance $\vec{x}$ is forced into $\lambda(s_{ij})$ rather than in the original leaf that is reached in $t_i$. 

\begin{definition}[Adversarial Gain]
Consider an instance $\vec{x}$ with true label $y$ and assume that $\vec{x}$ reaches the leaf $\lambda(s_o)$ when traversing the tree $t_i$ upon prediction. For any leaf $\lambda(s_{ij}) \in L(t_i,\vec{x})$, we define the \emph{adversarial gain} as follows:
\[
G(t_i,\vec{x},y,\lambda(s_{ij})) =
\begin{cases}
s_o - s_{ij} & \text{if } y = +1 \\
s_{ij} - s_o & \text{if } y = -1 \\
\end{cases}.
\]
\end{definition}

Intuitively, a leaf has a positive adversarial gain whenever it moves the instance closer to the wrong class than the original prediction. For example, consider again the decision tree in Figure~\ref{fig:tree} and the instance $\langle 9.1,0.1 \rangle$ with label -1. This instance originally reaches the leaf $\lambda_1(0.2)$ upon prediction; if the instance instead reaches the leaf $\lambda_4(0.6)$ as the result of adversarial manipulations, the adversarial gain is $0.6 - 0.2 = 0.4$, because an attack moving from $\lambda_1(0.2)$ to $\lambda_4(0.6)$ grants an advantage of 0.4 to the positive class over the negative class in terms of raw predictions. Notice that the definition of adversarial gain implicitly assumes that the inverse link function $\iota$ is monotonically increasing, because the true label determines whether the original raw prediction should be increased or decreased to lead to a prediction error.

We are finally ready to formulate our optimization problem. For each tree $t_i \in T$ and leaf $\lambda(s_{ij}) \in L(t_i,\vec{x})$, we introduce a new variable $z_{ij} \in \{0,1\}$. Then, determining the optimal evasion attack strategy for the instance $\vec{x}$ with true label $y$ against the ensemble $T$ can be done as follows.

\begin{problem}[Optimal Attack Strategy for Large-Spread Boosted Ensembles]\label{prob:optimization}
Determine the value assignment of the variables $\{z_{ij}\}_{ij}$ from the set $\{0,1\}$ which solves the following optimization problem:
\begin{align}
\textnormal{maximize} & \quad \sum_{i,j} z_{ij} \cdot G(t_i,\vec{x},y,\lambda(s_{ij})) \\
\textup{subject to} & \quad ||\sum_{i,j} z_{ij} \cdot \dist(\vec{x}, \lambda(s_{ij})) ||_p \leq k, \\
                    & \quad \forall i: \sum_j z_{ij} \leq 1.
\end{align}

In the following, we let $\Gamma$ stand for the optimal value of the objective function taken by the identified solution.
\end{problem}

In words, the attacker wants to maximize the total adversarial gain (1), which is the best strategy to force the wrong prediction, under two constraints: (2) the $L_p$-norm of the adversarial perturbation required to perform the attack is bounded above by the maximum perturbation $k$, and (3) just a single attack per tree is chosen, because a single leaf of each tree is reached upon prediction.

Very importantly, note that constraint (2) leverages the observation that the optimal evasion attack against a large-spread ensemble can be decomposed into a sum of adversarial perturbations operating against the individual trees~\cite{CalzavaraCPP23}. Moreover, observe that the optimal attack strategy may not be successful in flipping the ensemble prediction. Once a solution to the optimization problem is found, it is possible to verify robustness in two ways, as illustrated in the following.

\subsection{Basic Verification Algorithm}
The first version of the verification algorithm explicitly constructs the smallest perturbation $\vec{\delta}_{opt}$ that might lead to an attack and then checks whether it is successful or not. In particular, given a large-spread boosted ensemble $T$ and an instance $\vec{x}$ with true label $y$, the basic algorithm (``$BV$'', henceforth)
verifies robustness against $A_{p,k}$ as follows:
\begin{enumerate}
    \item Let $T(\vec{x}) = y'$. If $y' \neq y$, return False. Otherwise, solve Problem~\ref{prob:optimization} for the given input and initialize the adversarial perturbation $\vec{\delta}_{opt}$ to the vector $\langle 0, \ldots, 0 \rangle$.
    \item For each $z_{ij} = 1$, add the vector $\dist(\vec{x}, \lambda(s_{ij}))$ to the adversarial perturbation $\vec{\delta}_{opt}$.
    \item If $T(\vec{x} + \vec{\delta}_{opt}) = y$ return True, otherwise return False.
\end{enumerate}

The following theorem formalizes the correctness of the basic verification algorithm.

\begin{theorem}
The basic verification algorithm $BV(T,\vec{x},y,p,k)$ returns True if and only if $T$ is robust on the instance $\vec{x}$ with true label $y$ against the attacker $A_{p,k}$.
\end{theorem}
\begin{proof}
We prove the two directions separately:
\begin{itemize}
    \item[$(\Rightarrow)$] Assume that $BV(T,\vec{x},y,p,k)$ returns True. In this case we know that $T(\vec{x}) = y$, hence we just need to prove that $\forall z \in A_{p,k}(\vec{x}): T(\vec{z}) = y$ to conclude. Since $BV(T,\vec{x},y,p,k)$ returns True, we know that $T(\vec{x} + \vec{\delta}_{opt}) = y$ for the adversarial perturbation $\vec{\delta}_{opt}$ constructed by the algorithm. Assume for simplicity that $y = +1$, a similar argument applies to the case $y = -1$. Since $T(\vec{x}) = +1$, we know that $\iota(\pred{T}(\vec{x})) \geq \tau$. We now formalize the following intuition: since $\iota$ is monotonically increasing, the best strategy to force a prediction error is lowering $\pred{T}(\vec{x})$ as much as possible to push it below $\tau$, so that the predicted class becomes -1. Since $T$ is large-spread, any evasion attack against $T$ can be decomposed into a sum of pairwise orthogonal perturbations $\vec{\delta}_1,\ldots,\vec{\delta}_m$ that do not affect the prediction paths of different trees when they are added together~\cite{CalzavaraCPP23}. Let $\vec{\delta} = \sum_{i = 1}^m \vec{\delta}_i$ and, for each $\vec{\delta}_i$, let $\lambda(s_{ij})$ be the leaf that it is reached by $t_i(\vec{x} + \vec{\delta}_i)$, which is the same leaf reached by $t_i(\vec{x} + \vec{\delta})$. This implies that the sum of the adversarial gains $G(t_i,\vec{x},y,\lambda(s_{ij}))$ is equal to $\pred{T}(\vec{x}) - \pred{T}(\vec{x} + \vec{\delta})$, i.e., such sum identifies how much $\pred{T}(\vec{x})$ can be lowered by $\vec{\delta}$. Note that we are implicitly using constraint (3) here, since we consider leaves belonging to different trees. Now observe that $T(\vec{x} + \vec{\delta}_{opt}) = +1$ implies that $\iota(\pred{T}(\vec{x} + \vec{\delta}_{opt})) \geq \tau$. Since $\vec{\delta}_{opt}$ is the perturbation maximizing the sum of the adversarial gains, i.e., $\vec{\delta}_{opt}$ lowers $\pred{T}(\vec{x})$ as much as possible, we have that $\forall z \in A_{p,k}(\vec{x}): \pred{T}(\vec{z}) \geq \tau$, hence $\forall z \in A_{p,k}(\vec{x}): T(\vec{z}) = +1$.

    \item[$(\Leftarrow)$] Assume that $T$ is robust on $\vec{x}$. In this case we know that $T(\vec{x}) = y$, hence we just need to prove that $T(\vec{x} + \vec{\delta}_{opt}) = y$  for the adversarial perturbation $\vec{\delta}_{opt}$ constructed by the algorithm to conclude. Since $T$ is robust on $\vec{x}$, we know that $\forall z \in A_{p,k}(\vec{x}): T(\vec{z}) = y$. Hence, we just need to show that $\vec{x} + \vec{\delta}_{opt} \in A_{p,k}(\vec{x})$, which is equivalent to proving that $||\vec{\delta}_{opt}||_p \leq k$. This follows from constraint (2) by definition of $\vec{\delta}_{opt}$.
\end{itemize}
\end{proof}

\paragraph{Complexity.}
If $R$ is the time for solving Problem~\ref{prob:optimization}, the complexity of this algorithm is $R+O(dN)$,
because $O(dN)$ is the time to perform the additions in step (2) and step (3) is done in $O(N)$.


\subsection{Efficient Verification Algorithm}\label{sec:EV}
A more efficient robustness verification algorithm does not construct the optimal evasion attack $\vec{\delta}_{opt}$, but directly leverages the semantics of the adversarial gain. In particular, given a large-spread boosted ensemble $T$ and an instance $\vec{x}$ with true label $y$, the efficient algorithm (``$EV$'', henceforth) verifies robustness against $A_{p,k}$ as follows:
\begin{enumerate}
    \item Let $T(\vec{x}) = y'$. If $y' \neq y$, return False. Otherwise, solve Problem~\ref{prob:optimization} for the given input to determine the maximum value $\Gamma$ taken by function (1). 
    \item Let $\pred{T}(\vec{x}) = s$. If $y = +1$, return True if $\iota(s - \Gamma) \geq \tau$ and False otherwise. If $y = -1$, return True if $\iota(s + \Gamma) < \tau$ and False otherwise.
\end{enumerate}

\begin{figure*}[t]
    \centering
    \includegraphics[scale=0.5]{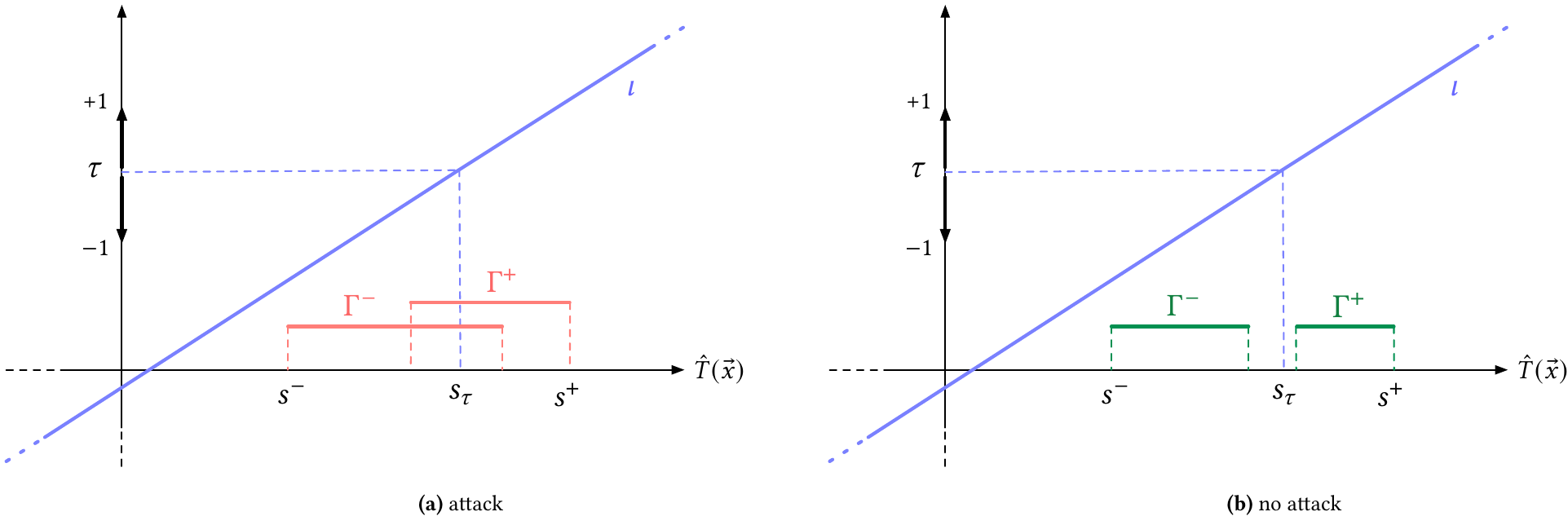}
    \caption{
        Correctness of robustness verification for the efficient algorithm $EV$.
    }
    \label{fig:gain}
\end{figure*}

\begin{theorem}
The efficient verification algorithm $EV(T,\vec{x},y,p,k)$ returns True if and only if $T$ is robust on the instance $\vec{x}$ with true label $y$ against the attacker $A_{p,k}$.
\end{theorem}
\begin{proof}
The key insight of the proof is shown in Figure~\ref{fig:gain}. The picture shows how the inverse link function $\iota$ (in blue) maps a raw prediction $\hat{T}(\vec{x})$ on the $x$-axis to a class label on the $y$-axis via the threshold $\tau$. The $\iota$ function shown here is just an example; it is not necessarily a linear function, but it can be an arbitrary monotonically increasing function. Let $s_{\tau}=\iota^{-1}(\tau)$ be the decision boundary. The model is \textit{not} robust on the instance $\vec{x}$ when scenario (a) happens. In this case, an attack is possible: for an instance with label $y=+1$, raw score prediction $\hat{T}(\vec{x})=s^{+}$ and maximum adversarial gain $\Gamma^{+}$, we have $s^{+}-\Gamma^{+} < s_{\tau}$, thus the instance is assigned the wrong class by $\iota$ as $\iota(s^{+}-\Gamma^{+}) < \tau$. Symmetrically, for an instance with label $y=-1$, raw score prediction $\hat{T}(\vec{x})=s^{-}$, and maximum adversarial gain $\Gamma^{-}$, we have $s^{-}+\Gamma^{-} \geq s_{\tau}$. Scenario (b) shows instead the case where no attack is possible and the model is robust, because the maximum adversarial gain is not sufficient to change the assigned label.
\end{proof}

\paragraph{Complexity.} 
If $R$ is the time for solving Problem~\ref{prob:optimization}, the complexity of this algorithm is $R+O(N)$ because computing $\pred{T}(\vec{x})$ takes $O(N)$ time. Note that this complexity is lower than the complexity of the previous basic verification algorithm.

\section{Solving the Optimization Problem}
\label{sec:solving}
The robustness verification algorithms presented in Section~\ref{sec:verification} build on a solution to Problem~\ref{prob:optimization}, so we now discuss how this problem can be solved. In the following, we focus on the efficient verification algorithm from Section~\ref{sec:EV}. Thus, we discuss how the maximum value $\Gamma$ of the objective function (1) can be computed for different values of $p$,
and omit how the actual value assignment of the
variables $z_{ij}$ is determined.

\subsection{Solution for $L_\infty$-Attackers}
We start from the case $L_\infty$, which is the easiest and most efficient to solve. Recall that $||\vec{x}||_{\infty} = \max \{|x_i| ~|~ 1 \leq i \leq d\}$.
We observe that for any set of pairwise orthogonal vectors $\{\vec{\delta}_i\}_i$ we have $|| \sum_i \vec{\delta}_i ||_\infty = \max \{||\vec{\delta}_i||_\infty\}$.
This means that constraint (2) is satisfied as long as we just consider leaves $\lambda(s_{ij})$ such that $|| \dist(\vec{x}, \lambda(s_{ij})) ||_\infty \leq k$, i.e., the leaves in the set $L(t_i,\vec{x})$. 
Since these are the only leaves considered in our optimization problem, constraint (2) is trivially satisfied by definition. Constraint (3) is also straightforward to enforce, because it just states that we can select at most one leaf per tree, hence the best attack strategy amounts to picking the leaf with the highest adversarial gain in each tree. In other words, to maximize function (1) we can then use the following algorithm:
\begin{enumerate}
    \item Initialize $\Gamma = 0$. For each tree $t_i \in T$, compute $L(t_i,\vec{x})$.
    \item For each $t_i \in T$, find the leaf $\lambda(s_{ij}) \in L(t_i,\vec{x})$ with the largest adversarial gain $G(t_i,\vec{x},y,\lambda(s_{ij}))$; in case of ties, arbitrarily break them to select one leaf. Then, sum $G(t_i,\vec{x},y,\lambda(s_{ij}))$ to $\Gamma$.
\end{enumerate}

\paragraph{Complexity.}
We recall that computing the sets $L(t_i,\vec{x})$ for all trees $t_i \in T$ takes $O(N)$ time. Since the leaf with the largest adversarial gain in each $t_i$ can be identified while computing $L(t_i,\vec{x})$, the total complexity of the algorithm is $R = O(N)$. It follows that also the $EV$ algorithm runs in $O(N)$ time, i.e., robustness verification can be performed in linear time with respect to the number of nodes of the ensemble.

\subsection{Solution for $L_0$-Attackers}
The case $L_0$ is more complicated. Again, recall that $||\vec{x}||_0 = |\{i \, | \, x_i \neq 0\}|$. We start from the observation that, for any set of pairwise orthogonal vectors $\{\vec{\delta}_i\}_i$, we have $|| \sum_i \vec{\delta}_i ||_0 = \sum_i || \vec{\delta}_i ||_0$. Hence, constraint (2) can be rewritten as:
\[
\sum_{i,j} z_{ij} \cdot ||\dist(\vec{x}, \lambda(s_{ij}))||_0 \leq k.
\]

With this insight, we observe that our optimization problem is reminiscent of a variant of the classic 0-1 \textit{knapsack problem}~\cite{pisinger1998knapsack}, where the adversarial gain represents the value of the items and the $L_0$-norm of the adversarial perturbations represents their weight.
We first review the knapsack problem and how it can be efficiently solved with dynamic programming.

Given a set of items $S$, each with a value $v$ and a weight $w$, and a maximum capacity $W$, the goal of 0-1 knapsack is to choose the best items to pick to maximize the overall value without exceeding the capacity $W$.
Formally, for each item $h \in S$ we introduce a new variable $z_h \in \{0,1\}$, and we formulate 0-1 knapsack as the optimization problem:
\begin{align*}
\textnormal{maximize} & \quad \sum_h z_h \cdot v_h \\
\textnormal{subject to} & \quad \sum_h z_h \cdot w_h \leq W.
\end{align*}

Assuming that all weights are positive integers, we can build a matrix $M$ of size $(|S|+1) \times (W+1)$, where the entry $M[h,w]$ stores the maximum value that can be obtained with capacity less than or equal to $w$ using up to $h$ items.
The matrix $M$ can be defined as follows:
\begin{itemize}
\item $M[0,w] = 0$ for all $0 \leq w \leq W$, i.e., the maximum value is 0 when the current maximum weight $w$ is 0 and no item can be taken yet.
\item $M[h,w] = M[h-1,w]$ when $w \geq 1$ and $w_h > w$, i.e., if the weight of item $h$ exceeds the current maximum weight $w$, then $h$ cannot be taken.
\item $M[h,w] = \max \{M[h-1,w], M[h-1,w-w_h] + v_h\}$ when $w \geq 1$ and $w_h \leq w$, i.e., if the weight of item $h$ does not exceed the current maximum weight $w$, then $h$ can either be taken or not, depending on its value.
\end{itemize}

The solution to the problem is then found in the lower right corner of the matrix $M$. This simple dynamic programming algorithm has complexity $O(|S| \cdot W)$.

We now discuss how the same idea can be generalized to efficiently solve Problem~\ref{prob:optimization}. For each tree $t_i \in T$ and leaf $\lambda(s_{ij}) \in L(t_i,\vec{x})$ such that $G(t_i,\vec{x},\lambda(s_{ij})) > 0$, we compute $||\dist(\vec{x}, \lambda(s_{ij}))||_0$. Note that the value of such computation is always a positive integer by definition of $||\cdot||_0$. The key difference of Problem~\ref{prob:optimization} with respect to 0-1 knapsack is that we have to enforce constraint (3), which ensures that just a single attack per tree is chosen.
To deal with this, we create a matrix $M$ of size $(m+1) \times (k+1)$ defined as follows:
\begin{itemize}
\item $M[0,q] = 0$ for all $0 \leq q \leq k$.
\item $M[i,q] = M[i-1,q]$ when $i \geq 1$ and $\forall \, \lambda(s_{ij}) \in L(t_i,\vec{x}): ||\dist(\vec{x}, \lambda(s_{ij}))||_0 > q$.
\item $M[i,q] = \max \{M[i-1,q], \max Q\}$ when $i \geq 1$ and $\exists \, \lambda(s_{ij}) \in L(t_i,\vec{x}): ||\dist(\vec{x}, \lambda(s_{ij}))||_0 \leq q$, where
\begin{align*}
Q = & \{M[i-1,q - ||\dist(\vec{x}, \lambda(s_{ij}))||_0] + G(t_i,\vec{x},y,\lambda(s_{ij})) ~|~ \\
& \quad \quad \lambda(s_{ij}) \in L(t_i,\vec{x}) \wedge ||\dist(\vec{x}, \lambda(s_{ij}))||_0 \leq q\}.
\end{align*}
\end{itemize}

The idea of the solving algorithm is equivalent to the dynamic programming approach for 0-1 knapsack, where the second clause encodes the case where the tree $t_i$ cannot be attacked within the current maximum perturbation $q$ and the third clause instead encodes the case where $t_i$ can be attacked, with the algorithm determining whether this should be done or not. Note that the third clause admits the possibility of having a set of possible attacks against $t_i$, in which case we select the one maximizing the adversarial gain against the ensemble $t_1,\ldots,t_i$. Soundness comes from the observation that, in a large-spread ensemble, each feature can be successfully attacked in at most one tree. Hence, an optimal attack against the ensemble $t_1,\ldots,t_{i-1}$ cannot be invalidated when considering the next tree $t_i$, because the features manipulated by an attack against $t_i$ cannot affect the prediction paths of $t_1,\ldots,t_{i-1}$, i.e., the verification problem has an optimal sub-structure. The solution to the problem $\Gamma$ is again found in the lower right corner of the matrix $M$, i.e., $M[m][k]$.

\paragraph{Complexity.}
We compute the sets $L(t_i,\vec{x})$ for all trees $t_i \in T$ in $O(N)$ time. The computation of $\max Q$ takes $O(2^D)$ time, where $D$ is the maximal tree depth in the ensemble. The total complexity of the algorithm is therefore $R = O(m \cdot k \cdot 2^D)$ since $N=O(m \cdot 2^D)$. It follows that also the $EV$ algorithm runs in this time. Technically speaking, the complexity of the algorithm is \emph{pseudo-polynomial}, because there is no formal guarantee that $k$ is bounded by a polynomial function of $m \cdot 2^D$. Nevertheless, $k$ is very small compared to $m \cdot 2^D$ in practical cases, because the number of features that a meaningful $L_0$-norm attacker can perturb is much smaller than the size of the ensemble. Indeed, observe that an $L_0$-norm attacker who can arbitrarily corrupt $k$ features can trivially break the robustness of any classifier when $k$ grows large.

\subsection{Solution for $L_p$-Attackers}
Finally, we deal with the case $L_p$ with $p \in \N$. Again, we start from a simple observation: for any set of pairwise orthogonal vectors $\{\delta_i\}_i$ and any $p \in \N$, we have $|| \sum_i \vec{\delta}_i ||_p = ( \sum_i || \vec{\delta}_i ||_p^p )^{1/p}$. Hence, constraint (2) can be rewritten as:
\[
\sum_{i,j} z_{ij} \cdot ||\dist(\vec{x}, \lambda(s_{ij}))||_p^p \leq k^p.
\]

We can then leverage the same idea of the case $L_0$, but some additional care is needed. In particular, observe that $||\dist(\vec{x}, \lambda(s_{ij}))||_p^p$ is not necessarily a positive integer, so the previous formulation cannot be readily applied. A possible solution to work again with positive integers is to multiply each $||\dist(\vec{x}, \lambda(s_{ij}))||_p^p$ by a suitable power of 10. In particular, let $\eta$ stand for a normalization function which multiplies its argument by $10^{\ell}$, where $\ell \geq 0$ is a constant large enough to ensure that all the arguments of $\eta$ are transformed into integers, e.g., if $\ell = 3$ and $x = 0.123$, then $\eta(x) = 123$. We can therefore reuse the same formulation we described for the case $L_0$, where we replace $||\dist(\vec{x}, \lambda(s_{ij}))||_0$ with $\eta(||\dist(\vec{x}, \lambda(s_{ij}))||_p^p)$ and $k$ with $\eta(k^p)$.

\paragraph{Complexity.}
Similarly to the case $L_0$, it follows that the time for solving Problem~\ref{prob:optimization} is $R = O(m \cdot \eta(k^p) \cdot 2^D)$, which is again pseudo-polynomial, but observe that the term $\eta(k^p)$ may grow very fast, because the normalization factor $\ell$ appears as an exponent of 10. Still, when $\ell$ is small, one can directly apply the proposed approach to establish exact robustness bounds. In practice, $\ell$ can be made small by performing a discretization of the feature space, which is a common pre-processing practice in real-world scenarios. Even in absence of discretization, one may mitigate the complexity by enforcing a properly reduced value of $\ell$ in the verification phase. By taking $\lfloor \eta( ||\dist(\vec{x}, \lambda(s_{ij}))||_p^p ) \rfloor$ and $\lceil \eta( k^p ) \rceil$ in constraint (2), we get a conservative approximation for robustness verification, because distances are under-approximated and the maximum adversarial perturbation is over-approximated with an arbitrarily chosen precision. This way, one may occasionally get false alarms where robust instances are flagged as not robust, but no evasion attack is missed.

\subsection{NP-Hardness Result}
We proposed verification algorithms for the case $p \in \extendedN$ based on 0-1 knapsack, hence running in pseudo-polynomial time. Here we show that such limitation is fundamental, because there exists no polynomial time algorithm for the case $p \in \extendedN$ despite our large-spread requirement.

\begin{theorem}
The robustness verification problem for large-spread boosted ensembles is NP-hard when considering attackers based on the $L_p$-norm, for any $p \in \extendedN$.
\end{theorem}
\begin{proof}
We show a reduction from the Subset Sum Problem (SSP)~\cite{10.5555/574848} to robustness verification for large-spread boosted ensembles. We consider the variant of SSP where all inputs are positive, which is still NP-hard.
For simplicity, we define the reduction for the case $p = 1$ and we then explain how the construction can be generalized to an arbitrary $p \in \extendedN$.

Let $\mathbb{S}$ be a multiset of integers and let $\mathbb{G}$ be an integer, the goal of SSP is determining whether there exists any subset of $\mathbb{S}$ which sums to $\mathbb{G}$. Let $\mathbb{S} = \{s_1,\ldots,s_m\}$ be a multiset of $m$ integers and $\zeta = \frac{1}{m+1}$.
We construct a large-spread boosted ensemble $T$ with $m$ trees such that $T$ is robust against $A_{1,\mathbb{G}}$ on the instance $\vec{x} = \langle \zeta, \ldots, \zeta \rangle$ with true label $-1$ if and only if there does \emph{not} exist a subset of $\mathbb{S}$ which sums to $\mathbb{G}$. Each tree $t_i$ is a regression stump, i.e., a regression tree of depth 1, built on top of a different feature $i$. Specifically, $t_i = \sigma(i,s_i,\lambda(0),\lambda(s_i))$. We stipulate that the inverse link function $\iota$ used by the ensemble $T$ is the identity function and that the threshold $\tau$ is set to $\mathbb{G}$.

We have that $\pred{T}(\vec{x}) = 0 < \tau$, hence $T(\vec{x}) = -1$. This means that $T$ is robust on $\vec{x}$ if and only if there does not exist any $\vec{z} \in A_{1,\mathbb{G}}(\vec{x})$ such that $\pred{T}(\vec{z}) \geq \tau$. Since $T$ is large-spread, finding such $\vec{z}$ is equivalent to finding pairwise orthogonal perturbations $\vec{\delta}_1,\ldots,\vec{\delta}_m$ such that $||\sum_{i = 1}^m \vec{\delta}_i||_1 \leq \mathbb{G}$ and $\sum_{i = 1}^m t_i(\vec{x} + \vec{\delta}_i) \geq \tau = \mathbb{G}$. Let $\{\vec{\delta}_j\}_j$ be the multiset of the non-zero perturbations, we now observe the following inequalities.
\begin{enumerate}
\item $||\sum_{i = 1}^m \vec{\delta}_i||_1 = ||\sum_j \vec{\delta}_j||_1 = \sum_j ||\vec{\delta}_j||_1 = \sum_j (s_j - \zeta) \leq \mathbb{G}$
\item $\sum_{i = 1}^m t_i(\vec{x} + \vec{\delta}_i) = \sum_j t_j(\vec{x} + \vec{\delta}_j) = \sum_j s_j \geq \mathbb{G}$
\end{enumerate}

By using the first inequality, we obtain:
\[
\sum_j s_j \leq \mathbb{G} + \sum_j \zeta \leq \mathbb{G} + m \cdot \zeta < \mathbb{G} + 1.
\]
Combined with the second inequality, this leads to:
\[
\mathbb{G} \leq \sum_j s_j < \mathbb{G} + 1, 
\]
which implies $\sum_j s_j = \mathbb{G}$, because all the elements of the summation are integers. We conclude that $T$ is robust on $\vec{x}$ if and only if there does not exist $\{s_j\}_j$ such that $\sum_j s_j = \mathbb{G}$, i.e., SSP does not have a solution. Of course, this reduction operates in polynomial time.

The reduction generalizes to the case $p > 1$ by observing that, when working with large-spread ensembles, the norm of the sum of adversarial perturbations can be related to the sum of the their norms. In particular, for any set of pairwise orthogonal perturbations $\{\delta_i\}_i$ and any $p \in \N$, we have $|| \sum_i \vec{\delta}_i ||_p = ( \sum_i || \vec{\delta}_i ||_p^p )^{1/p}$. Hence, the reduction can be performed by making the following changes: set $t_i = \sigma(i,s_i^p,\lambda(0),\lambda(s_i^p))$, increase the maximum adversarial perturbation from $\mathbb{G}$ to $\mathbb{G}^p$, and similarly set the threshold $\tau = \mathbb{G}^p$. 

Finally, the case $p = 0$ uses a different reduction, but the underlying intuition remains the same. The constructed large-spread ensemble includes regression trees, rather than regression stumps. For each integer $s_i \in \mathbb{S}$ we generate $s_i$ features, call them $f^i_1,\ldots,f^i_{s_i}$. The corresponding regression tree $t_i$ is a right chain of depth $s_i$, where we inductively define the sub-tree rooted at layer $j$ as $t_j' = \sigma(f^i_j,1,\lambda(0),t_{j+1}')$ with the base case $t_{s_i}' = \sigma(f^i_{s_i},1,\lambda(0),\lambda(s_i))$. This way, a successful evasion attack against $t_i$ requires a perturbation $\vec{\delta}_i$ such that $||\vec{\delta}_i||_0 = s_i$ and the corresponding adversarial gain is $s_i$. Note that, in this setting, for each tree there is only one relevant leaf $\lambda(s_i)$ to consider, i.e., the leaf providing a positive adversarial gain $s_i$. Since we do not need to materialize every leaf in $t_i$, the reduction is performed in polynomial time. Correctness follows by the same argument used for the case $p = 1$.
\end{proof}

\section{Implementation}
\label{sec:implementation}
We here describe the implementation of {\tool}, our robustness verification tool for large-spread boosted ensembles. We then present a simple training algorithm for large-spread boosted ensembles, implemented on top of the popular LightGBM library. We will make our implementation available on GitHub upon paper acceptance.

\paragraph{\tool.}
{\tool} is a C++ implementation of the $EV$ algorithm presented in Section~\ref{sec:EV}, which uses the techniques in Section~\ref{sec:solving} to solve the underlying optimization problem. It currently supports large-spread boosted ensembles trained using LightGBM, but it can be easily extended to other input formats, e.g., to support XGBoost.

\paragraph{Training Algorithm.}
We implemented a training algorithm for large-spread boosted ensembles on top of the popular learning algorithm for boosted gradient tree ensembles available in LightGBM. In particular, we extended LightGBM to keep track of the thresholds used for each feature upon tree construction. Once the threshold $v$ has been chosen for feature $f$ in a tree $t_i$, all the other thresholds $v'$ for the same feature which would violate the large-spread condition (Definition~\ref{def:spread}) are not considered for node splitting when constructing the trees $t_j$ with $j > i$. Note that this may effectively prevent the reuse of $f$ within the next trees, in particular when all the other thresholds $v'$ are close to $v$. This ``black-listing'' method reduces, tree by tree, the available splitting options for creating new nodes. Eventually, this may force the training algorithm to stop the construction of new trees. In the experimental section we show that this worst-case scenario is not detrimental in terms of accuracy of the generated model.
\section{Experimental Evaluation}
\label{sec:experiments}
In this section we experimentally evaluate the distinctive traits of large-spread boosted ensembles with respect to accuracy and robustness. Moreover, we evaluate the primary objective of our large-spread models, i.e., we assess whether they are amenable to efficient robustness verification.

\subsection{Methodology}
Our experimental evaluation is performed on three public datasets: Fashion-MNIST \footnote{\url{https://www.openml.org/search?type=data\&sort=runs\&id=40996&status=active}} (FMNIST), MNIST\footnote{\url{https://www.openml.org/search?type=data\&sort=runs\&id=554}} and Webspam\footnote{\url{https://www.csie.ntu.edu.tw/~cjlin/libsvmtools/datasets/binary.html}}.
We reduce FMNIST and MNIST to binary classification tasks by considering two subsets of them. In particular, for FMNIST we consider the instances with class 0 (T-shirt/top) and 3 (Dress), while for MNIST we keep the instances representing the digits 2 and 6. The key characteristics of the chosen datasets are reported in Table~\ref{tab:datasets}. Each dataset is partitioned into a training set, validation set and a test set, using 55/15/30 stratified random sampling as customary. Additionally, we reduce the size of the Webspam test set to match the number of instances of the FMNIST test set using stratified random sampling.

We then use the following methodology:
\begin{enumerate}
    \item We perform hyper-parameter tuning to identify the best-performing boosted tree ensemble trained using LightGBM. Evaluation is performed on the validation set looking for the ensemble with highest accuracy. We train up to 500 regression trees, setting an early stopping criterion of 50 boosting rounds. We tune the number of leaves from the set $\{16,32,64,128,256\}$, setting the learning rate to 0.1.
    
    \item We perform the same hyper-parameter tuning to identify the best-performing large-spread boosted tree ensemble trained using our algorithm. We do this for different norms ($L_\infty,L_0,L_1$) and different magnitudes of the adversarial perturbation $k$. The value of $k$ is dataset-specific, because the data distribution inherently affects robustness, hence perturbations which are meaningful for a given dataset may be too weak or too strong for another dataset.
     
    \item We assess the accuracy and robustness of both the traditional boosted ensemble and our large-spread ensemble to understand how the enforced model restriction impacts on classic performance measures in different settings (different norms and magnitudes of adversarial perturbations).

    \item We finally assess the efficiency and scalability of robustness verification with respect to existing state-of-the-art tools for different norms.
\end{enumerate}

For all the experiments, we fix the inverse link function $\iota$ to the identity and we set the threshold $\tau$ to 0.

\paragraph{Baselines.}
Robustness verification for large-spread models can be efficiently performed using \tool, while traditional models without the large-spread restriction must be analyzed using existing verification tools. We consider two approaches as baselines.
For $L_\infty$ we consider SILVA~\cite{RanzatoZ20}, a state-of-the-art verification tool for tree ensembles. SILVA uses abstract interpretation and in particular the hyper-rectangle abstract domain to perform exact robustness verification against attackers based on the $L_\infty$-norm.
For $L_0$ and $L_1$, we instead consider an exact approach based on mixed integer linear programming (MILP)~\cite{KantchelianTJ16}.

Since SILVA and MILP may not scale on the entire test set, we enforce a timeout of 10 seconds per instance and we compute approximate values of robustness as required. In particular, when the verification tool goes in timeout, we consider its output as unknown, and we compute lower and upper bounds to robustness by considering such unknown instances as non-robust or robust respectively. We present approximate results using the $\pm$ notation.

\begin{table}[t]
    \centering
    \begin{tabular}{c|c|c|c}
    \toprule
    \textbf{Dataset} & \textbf{Instances} & \textbf{Features} & \textbf{Class Distribution} \\
    \midrule
    FMNIST & 13,866 & 784 & $50\% / 50\%$ \\
    MNIST & 14,000 & 784 & $51\% / 49\%$ \\
    Webspam & 350,000 & 254 & $70\% / 30\%$ \\
    \bottomrule
    \end{tabular}
\caption{Dataset statistics}
\label{tab:datasets}
\end{table}

\subsection{Accuracy and Robustness}
Table~\ref{tab:measures} compares the accuracy and the robustness of traditional GBDT models trained using LightGBM and our large-spread boosted ensembles trained using our algorithm.

\begin{table*}[t]
    \centering
    \begin{tabular}{c|c|c|c|c|c|c}
    \toprule
    \multirow{2}{*}{\textbf{Dataset}} & \multirow{2}{*}{\textbf{$p$}} & \multirow{2}{*}{\textbf{$k$}} & \multicolumn{2}{c}{\textbf{Accuracy}} & \multicolumn{2}{c}{\textbf{Robustness}} \\
    \cmidrule{4-7}
    & & & GBDT & Large-Spread & GBDT & Large-Spread \\
    \midrule
    \multirow{9}{*}{FMNIST} & \multirow{3}{*}{$\infty$} & 0.005 & \revise{0.97} & \revise{0.97} & \revise{0.93 $\pm$ 0.01} & \revise{0.96}\\
    & & 0.01 & \revise{0.97} & \revise{0.97} & \revise{0.81 $\pm$ 0.12} & \revise{0.89} \\
    & & 0.015 & \revise{0.97} & \revise{0.97} & \revise{0.68 $\pm$ 0.21} & \revise{0.88} \\
    \cmidrule{2-7}
    & \multirow{3}{*}{0} & 1 & \revise{0.97} & \revise{0.96} & \revise{0.94} & \revise{0.96} \\
    & & 2 & \revise{0.97} & \revise{0.96} & \revise{0.90} & \revise{0.90}\\
    & & 3 &  \revise{0.97} & \revise{0.96} & \revise{0.86} & \revise{0.85}\\
    \cmidrule{2-7}
    & \multirow{3}{*}{1} & 0.05 & \revise{0.97} & \revise{0.97} & \revise{0.87} & \revise{0.89}\\
    & & 0.1 & \revise{0.97} & \revise{0.96} & \revise{0.81} & \revise{0.83}\\
    & & 0.15 & \revise{0.97} & \revise{0.96} & \revise{0.72} & \revise{0.66}\\
    \midrule
    \multirow{9}{*}{MNIST} & \multirow{3}{*}{$\infty$} & 0.01 & \revise{0.99} & \revise{0.99} & \revise{0.98} & \revise{0.99} \\
    & & 0.02 & \revise{0.99} & \revise{0.99} & \revise{0.93 $\pm$ 0.02} & \revise{0.98} \\
    & & 0.03 & \revise{0.99} & \revise{0.99} & \revise{0.85 $\pm$ 0.10} & \revise{0.97} \\
    \cmidrule{2-7}
    & \multirow{3}{*}{0} & 1 & \revise{0.99} & \revise{0.99} & \revise{0.77 $\pm$ 0.21} & \revise{0.99} \\
    & & 2 & \revise{0.99} & \revise{0.99} & \revise{0.67 $\pm$ 0.21} & \revise{0.98}\\
    & & 3 & \revise{0.99} & \revise{0.99} & \revise{0.44 $\pm$ 0.21} & \revise{0.96}\\
    \cmidrule{2-7}
    & \multirow{3}{*}{1} & 0.05 & \revise{0.99} & \revise{0.99} & \revise{0.87 $\pm$ 0.05} & \revise{0.99} \\
    & & 0.1 & \revise{0.99} & \revise{0.99} & \revise{0.86 $\pm$ 0.04} & \revise{0.96} \\
    & & 0.15 & \revise{0.99} & \revise{0.99} & \revise{0.84 $\pm$ 0.04} & \revise{0.84}\\
    \midrule
    \multirow{9}{*}{Webspam} & \multirow{3}{*}{$\infty$} & 0.0004 & \revise{0.99} & \revise{0.99} & \revise{0.90 $\pm$ 0.08} & \revise{0.97} \\
    & & 0.0006 & \revise{0.99} & \revise{0.99} & \revise{0.84 $\pm$ 0.13} & \revise{0.97} \\
    & & 0.0008 & \revise{0.99} & \revise{0.99} & \revise{0.79 $\pm$ 0.18} & \revise{0.96} \\
    \cmidrule{2-7}
    & \multirow{3}{*}{0} & 1 & \revise{0.99} & \revise{0.96} & \revise{-} & \revise{0.91} \\
    & & 2 & \revise{0.99} & \revise{0.96} & \revise{-}& \revise{0.86}\\
    & & 3 & \revise{0.99} & \revise{0.96} & \revise{-} & \revise{0.79}\\
    \cmidrule{2-7}
    & \multirow{3}{*}{1} & 0.002 & \revise{0.99} & \revise{0.97} & -  & \revise{0.95} \\
    & & 0.003 & \revise{0.99} & \revise{0.96} & -  & \revise{0.93} \\
    & & 0.004 & \revise{0.99} & \revise{0.96} & - & \revise{0.89}\\
    \bottomrule
    \end{tabular}
\caption{Accuracy and robustness (against $A_{p,k}$) measures for traditional GBDT models and large-spread boosted ensembles. Robustness of GBDT over Webspam for the $p \in \{0,1\}$ case is omitted, because MILP does not scale to those models.}
\label{tab:measures}
\end{table*}

\paragraph{Results for $L_\infty$-Attackers.}
We observe that the large-spread condition preserves the accuracy of the trained models: in all the cases, the accuracy of large-spread boosted ensembles is equivalent to the accuracy of GBDT models.
Although prior work already showed that large-spread ensembles can be accurate, some performance degradation was observed~\cite{CalzavaraCPP23}. Our use of boosting further relaxes the practical limitations of the large-spread condition, likely because the trees of a boosted ensemble are specifically trained to compensate the weaknesses of their predecessors, yielding models which are equivalent to the state of the art. In turn, the large-spread condition is beneficial to robustness verification. Indeed, a state-of-the-art tool like SILVA can only provide approximate robustness bounds, because verification takes too much time. The bounds are rather large and make robustness verification unreliable: for many cases, the uncertainty is above $\pm 0.10$. For example, in the case of the Webspam dataset with the highest perturbation, robustness ranges between 0.61 and 0.98. This gap is unacceptable for a credible security verification. \tool{} is instead able to establish exact robustness bounds for our large-spread models. Observe that, in all cases, the robustness of large-spread models is close to the most optimistic robustness estimate of the GBDT models and quite close to 1 in general.

\paragraph{Results for $L_0$-Attackers.}
In this case, the large-spread restriction introduces a very limited accuracy loss with respect to the GBDT models. The highest loss is on the Webspam dataset, where the large-spread models sacrifice 0.03 of accuracy. This is definitely acceptable, because their accuracy is still 0.96, which is close to 1. As to robustness, we observe that MILP struggles against the analyzed models. In particular, for the Webspam dataset it is completely unable to provide a reasonable robustness estimate, because more than 90\% of the instances go in timeout after 10 seconds. Again, \tool{} is able to establish exact robustness bounds very efficiently in all cases, showing that the robustness of the large-spread models is reasonably high. In particular, even against the strongest attacker, their robustness is still above the frequency of the majority class (0.70). For the other two datasets, the robustness of large-spread models is either comparable to the most optimistic robustness estimate of the GBDT models or even much better. The major improvement is for the MNIST dataset, where robustness against the strongest attacker increases from 0.65 (at best) to 0.96 when moving from GBDT models to large-spread models.

\paragraph{Results for $L_1$-Attackers.}
Also the evaluation against $L_1$-attackers is largely positive. Again the large-spread condition does not enforce any significant loss of accuracy compared to GBDT models, because the highest loss is 0.03 (again on the Webspam dataset). The robustness of the large-spread models is generally higher than the robustness of GBDT models or equal to the most optimistic estimate, with just a few exceptions. These cases happen when the best-performing large-spread model is noticeably smaller than the best-performing GBDT model, hence empirically less robust against evasion attacks. This can be easily voided by integrating robustness verification directly in the hyper-parameter tuning pipeline of large-spread models, thus looking for the model showing highest robustness or the best trade-off between accuracy and robustness. Note that this cannot be easily done for GBDT models, because robustness verification does not always scale.

\subsection{Performance and Scalability}
We perform some additional experiments to assess the benefits on robustness verification enabled by our large-spread boosted ensembles.

\paragraph*{Performance.}
In the first experiment, we train a traditional ensemble of the same size (number of trees and nodes) of our best-performing large-spread ensemble and we compare the robustness verification times of SILVA (for $L_\infty$) and MILP (for $L_0,L_1$) against \tool. We limit the analysis to 500 instances of the test set sampled using stratified random sampling, because our competitors often take too much time to run.\footnote{Using the full test set can only improve the picture in favour of \tool, in particular for the Webspam dataset, whose full test set includes around 100,000 instances; this sheer size would magnify the difference in the measured verification times.}
The results are shown in Table~\ref{tab:time}.

The first observation we make is that \tool{} is significantly faster than SILVA in the vast majority of cases. For example, for the FMNIST dataset and the highest perturbation, the verification times of \tool{} are three orders of magnitude lower than SILVA (4 seconds vs. 3,448 seconds). The same happens for the Webspam dataset, where \tool{} takes 4 seconds while SILVA takes 2,142 seconds for the highest perturbation. The MNIST dataset shows a different trend, because SILVA is very efficient there. We do not have a clear explanation for this observation, except that even NP-hard problems can be easy to solve in specific cases and for some reason the MNIST model seems straightforward to verify. However, we stress that SILVA does not provide formal complexity bounds, but reuses an abstract interpretation engine hoping that it scales to the NP-hardness of robustness verification. This is the case for MNIST, but the other two datasets make SILVA struggle. Finally, we discuss the comparison with MILP: again, \tool{} always performs better than its competitor. For example, for $L_0$-attackers on the Webspam dataset, verification time reduces from 902 seconds to 2 seconds, i.e., a reduction of two orders of magnitude. A similar picture can be drawn for $L_1$-attackers.

\begin{table}[t]
    \centering
    \begin{tabular}{c|c|c|c|c}
    \toprule
    \multirow{2}{*}{\textbf{Dataset}} & \multirow{2}{*}{\textbf{$p$}} & \multirow{2}{*}{\textbf{$k$}} & \multicolumn{2}{c}{\textbf{Verification Time}}\\
    \cmidrule{4-5}
    & & & Baseline & CARVE-GBM\\
    \midrule
    \multirow{9}{*}{FMNIST} & \multirow{3}{*}{$\infty$} & 0.005 & \revise{172} & \revise{12} \\
    & & 0.010 & \revise{290} & \revise{3} \\
    & & 0.015 & \revise{3,448} & \revise{4} \\
    \cmidrule{2-5}
    & \multirow{3}{*}{$0$} & 1 & \revise{23} & \revise{1} \\
    & & 2 & \revise{23} & \revise{1} \\
    & & 3 & \revise{23} & \revise{1} \\
    \cmidrule{2-5}
    & \multirow{3}{*}{$1$} & 0.05 & \revise{74} & \revise{8} \\
    & & 0.10 & \revise{37} & \revise{4} \\
    & & 0.15 & \revise{39} & \revise{4} \\
    \midrule
    \multirow{9}{*}{MNIST} & \multirow{3}{*}{$\infty$} & 0.01 & \revise{1} & \revise{13} \\
    & & 0.02 & \revise{4} & \revise{4} \\
    & & 0.03 & \revise{3} & \revise{3} \\
    \cmidrule{2-5}
    & \multirow{3}{*}{$0$} & 1 & \revise{23} & \revise{1} \\
    & & 2 & \revise{23} & \revise{1} \\
    & & 3 & \revise{23} & \revise{1} \\
    \cmidrule{2-5}
    & \multirow{3}{*}{$1$} & 0.05 & \revise{157} & \revise{8} \\
    & & 0.10 & \revise{42} & \revise{3} \\
    & & 0.15 & \revise{46} & \revise{4} \\
    \midrule
    \multirow{9}{*}{Webspam} & \multirow{3}{*}{$\infty$} & 0.0004 & \revise{216} &  \revise{5} \\
    & & 0.0006 & \revise{989} & \revise{4} \\
    & & 0.0008 & \revise{2,142} & \revise{4} \\
    \cmidrule{2-5}
    & \multirow{3}{*}{$0$} & 1 & \revise{902} & \revise{2} \\
    & & 2 & \revise{902} & \revise{2} \\
    & & 3 & \revise{902} & \revise{2} \\
    \cmidrule{2-5}
    & \multirow{3}{*}{$1$} & 0.002 & \revise{278} & \revise{5} \\
    & & 0.003 & \revise{179} & \revise{4} \\
    & & 0.004 & \revise{424} & \revise{5} \\
    \bottomrule
    \end{tabular}
\caption{Robustness verification times (in seconds) for the first 500 instances in the test set.}
\label{tab:time}
\end{table}

\begin{figure}[t]
    \includegraphics[width=.42\textwidth]{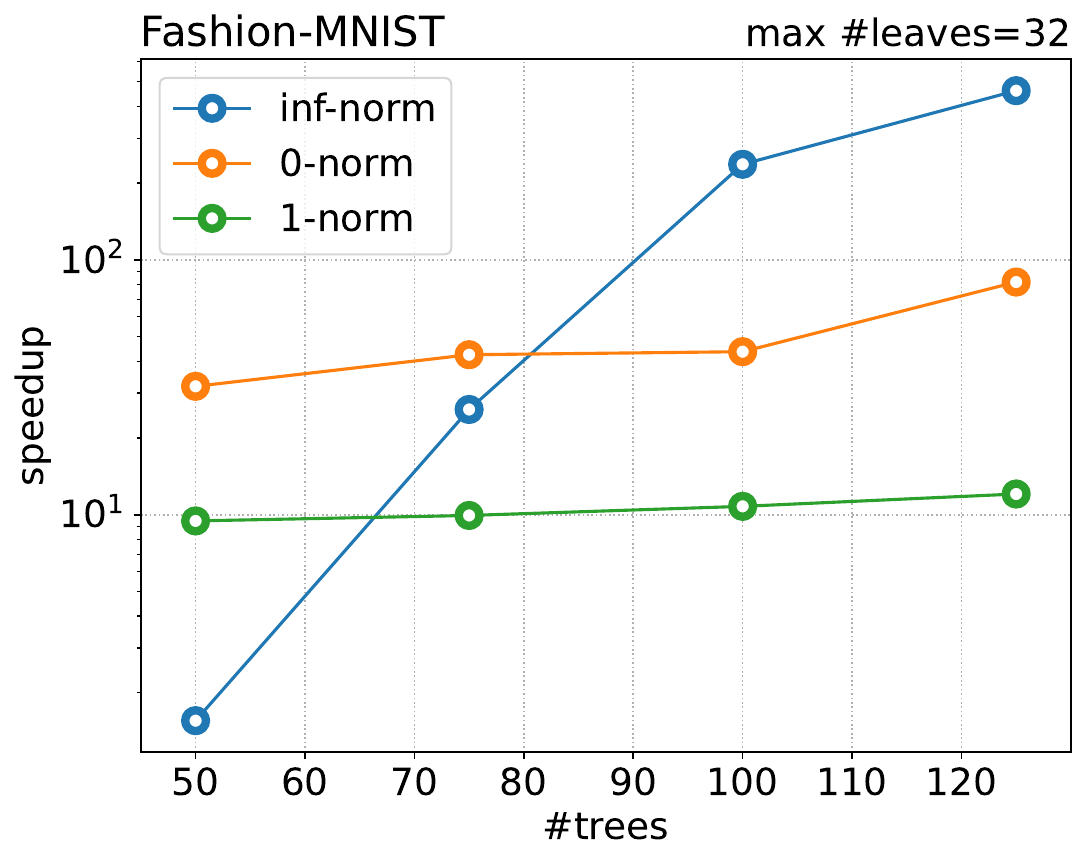}
  \caption{Speedup of the robustness verification time enabled by the use of \tool{} over competitors.
  }
  \label{fig:efficiency}
\end{figure}

\paragraph*{Scalability.}
In our second experiment, we assess the scalability of robustness verification by increasing the model size. In particular, we set the maximum number of leaves to 32 and we vary the number of trees in the ensemble trained on the FMNIST dataset to understand the trend of the robustness verification times. Results are shown in Figure~\ref{fig:efficiency}, where we plot the speedup of CARVE-GBM over competitors. As we can see, the speedup ranges from around 10 times to more than 400 times for the largest ensembles of 125 trees, i.e., verification times are reduced by at least one order of magnitude. Remarkably, this positive finding is also artificially penalized by the enforcement of a 30 seconds timeout per instance on SILVA and MILP. If these tools were allowed to run for a longer time, the comparison would be even more in favor of \tool, which is able to analyze each instance in less than one second.
\section{Related Work}
We already discussed that prior work studied the complexity of robustness verification for decision tree ensembles~\cite{KantchelianTJ16,WangZCBH20}. This problem is NP-complete for arbitrary $L_p$-norm attackers, even when considering decision stump ensembles~\cite{Andriushchenko019,WangZCBH20}. Despite this negative result, prior work proposed a plethora of different approaches to robustness verification for tree ensembles~\cite{TornblomN20,KantchelianTJ16,ChenZS0BH19,RanzatoZ20, CalzavaraFL20,DevosMD21, EinzigerGSS19, SatoKNO20}. Though effective in many cases, these techniques are bound to fail for large ensembles and complex datasets. We experimentally showed that state-of-the-art verification tools do not scale and are much less efficient than our approach based on large-spread ensemble training; thus, we generalize preliminary results on verifiable learning for simple random forest models~\cite{CalzavaraCPP23} to the gradient boosting setting.

Numerous studies in the literature have delved into novel algorithms for training tree ensembles that exhibit robustness against evasion attacks~\cite{CalzavaraLTAO20, ChenZBH19, KantchelianTJ16, Andriushchenko019, VosV21, RanzatoZ21, GuoTGZ22, VosV22-1, VosV22-2, Chen0JCJ21}. However, our work complements these endeavors. Our primary objective does not center around enforcing robustness; rather, robustness may emerge as a byproduct of our training algorithm. Our focus lies in facilitating efficient robustness verification for the trained models. Combining our approach with existing robust training algorithms may further improve their performance against evasion attacks.


It is essential to note that our work specifically addresses the classic definition of robustness, termed \emph{local} robustness in recent literature discussions on global robustness and related properties~\cite{Chen0QLJW21,CalzavaraCLMO22,LeinoWF21}. The aim of this line of research is to achieve security verification independent of the selection of a specific test set, thereby enhancing the credibility of security proofs. While acknowledging the popularity and relative ease of dealing with local robustness, we leave the extension of our framework to global robustness for future work.

Significant efforts have been invested in the robustness verification of deep neural networks (DNNs). Traditional approaches for exact verification, often based on Satisfiability Modulo Theories (SMT)\cite{KatzBDJK17, KatzHIJLLSTWZDK19, HuangKWW17} and integer linear programming\cite{BastaniILVNC16, LomuscioM17, TjengXT19, DuttaJST18}, frequently face scalability issues with large DNNs, similar to tree ensembles. To address these challenges, various proposals, such as pruning the original DNN~\cite{GuidottiLPT20} and identifying specific classes of DNNs for more efficient robustness verification~\cite{JiaR20}, have been presented. Xiao et al.~\cite{XiaoTSM19} introduced the concept of \textit{co-designing} model training and verification, emphasizing training models that balance accuracy and robustness to facilitate exact verification. While prior techniques offer empirical efficiency guarantees, our proposal stands out by providing a formal complexity reduction through the development of a polynomial time algorithm. Furthermore, our research focuses on tree ensembles rather than DNNs.

Lastly, recent work explored the adversarial robustness of model ensembles~\cite{YangLXK0L22}. The key finding of this study demonstrated that a combination of ``diversified gradient'' and ``large confidence margin'' serves as sufficient and necessary conditions for certifiably robust ensemble models. While this result may not directly apply to non-differentiable models like decision tree ensembles, the idea of diversifying models aligns with our large-spread condition. We plan to explore potential connections with this proposition in future work.
\section{Conclusion}
In this work, we generalized existing endeavours on verifiable learning from simple ensembles based on hard majority voting to state-of-the-art boosted ensembles, such as those trained using LightGBM~\cite{lightgbm}. We formally characterized robustness verification for large-spread boosted ensembles in terms of an optimization problem and we proposed efficient techniques to solve it in practical cases. We experimentally showed on public datasets that our verifiable learning techniques allows one to train models offering state-of-the-art accuracy, while being amenable to efficient robustness verification. Our analysis also confirmed that robustness verification for traditional tree ensembles does not scale when increasing the model size, thus reinforcing the importance of verifiable learning.

As future work, we would like to investigate different approaches to verifiable learning for tree ensembles, besides the large-spread condition advocated in existing work. Moreover, we plan to study how verifiable learning can be applied to fundamentally different classes of machine learning models, such as DNNs.

\bibliographystyle{plain}
\bibliography{main}

\end{document}